\documentclass[letterpaper]{article}
\usepackage{uai2019}
\usepackage[margin=1in]{geometry}

\usepackage{times}
\usepackage{smile}
\usepackage{color}

\usepackage{graphicx}

\def \Oind{s}
\def \OindLen{S}
\def \Iind{t}
\def \IindLen{m}
\def \tbtheta{\tilde{\btheta}}
\def \dd {\text{d}}

\newcommand{\la}{\langle}
\newcommand{\ra}{\rangle}
\def \Var {\text{Var}}
\def \consVariance{\Psi}
\def \consLip{G}
\def \consHes{M}
\def \consVarBound{\sigma^2}
\def \consVarWeight{W}
\def \consGrad{C_g}
\def \consGradLip{L_g}
\allowdisplaybreaks
\usepackage[compact]{titlesec}
\usepackage{sidecap}
\titlespacing*{\section}{0pt}{*0.1}{*0.1}
\titlespacing*{\subsection}{0pt}{*0.1}{*0.1}
\titlespacing*{\subsubsection}{0pt}{*0.1}{*0.1}

\title{An Improved Convergence Analysis of Stochastic Variance-Reduced Policy Gradient\thanks{\ \ To appear in the proceedings of the 35th International Conference on Uncertainty in Artificial Intelligence.}}

\author{} 

\author{ {\bf Pan Xu} \\
Department of Computer Science\\
University of California, Los Angeles\\
Los Angeles, CA 90095\\
\And
{\bf Felicia Gao} \\
Department of Computer Science\\
University of California, Los Angeles\\
Los Angeles, CA 90095\\
\And
{\bf Quanquan Gu}   \\
Department of Computer Science\\
University of California, Los Angeles\\
Los Angeles, CA 90095\\
}

\begin{document}
\maketitle

\begin{abstract}
We revisit the stochastic variance-reduced policy gradient (SVRPG) method proposed by \citet{papini2018stochastic} for reinforcement learning. We provide an improved convergence analysis of SVRPG and show that it can find an $\epsilon$-approximate stationary point of the performance function 
within $O(1/\epsilon^{5/3})$ trajectories. 
This sample complexity improves upon the best known result $O(1/\epsilon^2)$ by a factor of $O(1/\epsilon^{1/3})$. At the core of our analysis is (i) a tighter upper bound for the variance of importance sampling weights, where we prove that the variance can be controlled by the parameter distance between different policies; and (ii) a fine-grained analysis of the epoch length and batch size parameters such that we can significantly reduce the number of trajectories required in each iteration of SVRPG. We also empirically demonstrate the effectiveness of our theoretical claims of batch sizes on  reinforcement learning benchmark tasks. 
\end{abstract}

\section{INTRODUCTION}

Reinforcement learning (RL) is a sequential decision process that learns the best actions to solve a task by repeated, direct interaction with the environment \citep{sutton2018reinforcement}. In detail, an RL agent starts at one state and sequentially takes an action according to a certain policy, observes the resulting reward signal, and lastly, evaluates and improves its policy before it transits to the next state. A policy tells the agent which action to take at each state. Therefore, a good policy is critically important in a RL problem. Recently, policy gradient methods \citep{sutton2000policy} have achieved impressive successes in many challenging deep reinforcement learning applications \citep{kakade2002natural,schulman2015trust}, which directly optimizes the performance function $J(\btheta)$ (We will formally define it later) over a class of policies parameterized by some model parameter $\btheta$. In particular, policy gradient methods seek to find the best policy $\pi_{\btheta}$ that maximizes the expected return of the agent. They are generally more effective in the high-dimensional action space and enjoy the additional flexibility of stochasticity, compared with deterministic value-function based methods such as Q-learning and SARSA \citep{sutton2000policy}. 

In many RL applications, the performance function $J(\btheta)$ is non-concave and the goal is to find a stationary point $\btheta^*$ such that $\|\nabla J(\btheta^*)\|_2=0$ using gradient based algorithms. Due to the specialty of reinforcement learning, the objective function $J(\btheta)$ is calculated based on cumulative rewards arriving in a sequential way, which makes it impossible to calculate the full gradient directly. Therefore, most algorithms such as REINFORCE \citep{williams1992simple} and GPOMDP \citep{baxter2001infinite} need to actively sample trajectories to approximate the gradient $\nabla J(\btheta)$. This resembles the stochastic gradient (SG) based algorithms in stochastic optimization \citep{robbins1951stochastic} which require $O(1/\epsilon^2)$ trajectories to obtain $\EE[\|\nabla J(\btheta)\|_2^2]\leq\epsilon$ 
Due to the large variances caused by stochastic gradient, the convergence of SG based methods can be rather sample inefficient when the required precision $\epsilon$ is very small.  

To mitigate the negative effect of large variance on the convergence of SG methods, a large class of stochastic variance-reduced gradient (SVRG) algorithms were proposed for both convex \citep{johnson2013accelerating,xiao2014proximal,harikandeh2015stopwasting,nguyen2017sarah} and nonconvex \citep{allen2016variance,reddi2016stochastic,lei2017nonconvex,li2018simple,fang2018spider,zhou2018stochastic_nips} objective functions. SVRG has proved to achieve faster convergence in terms of the total number of stochastic gradient evaluations. These variance-reduced algorithms have since been applied to reinforcement learning in policy evaluation \citep{du2017stochastic}, trust-region policy optimization \citep{xu2017stochastic} and policy gradient \citep{papini2018stochastic}. In particular, \citet{papini2018stochastic} recently proposed a stochastic variance-reduced policy gradient (SVRPG) algorithm that marries SVRG to policy gradient for reinforcement learning. The algorithm saves on sample computation and improves the performance of the vanilla policy gradient methods based on SG. However, from a theoretical perspective, the authors only showed that SVRPG converges to a stationary point within $\EE[\|\nabla J(\btheta)\|_2^2]\leq\epsilon$ with $O(1/\epsilon^2)$ stochastic gradient evaluations (trajectory samples), which in fact only matches the sample complexity of SG based policy gradient methods. This leaves open the important question:

\textit{Can SVRPG be provably better than SG based policy gradient methods?}

We answer this question affirmatively and fill this gap between theory and practice in this paper. Specifically, we provide a sharp convergence analysis of SVRPG and show that it only requires $O(1/\epsilon^{5/3})$ stochastic gradient evaluations 
in order to converge to a stationary point $\btheta$ of the performance function, i.e., $\EE[\|\nabla J(\btheta)\|_2^2]\leq\epsilon$. This sample complexity of SVRPG is strictly lower than that of SG based policy gradient methods by a factor of $O(1/\epsilon^{1/3})$. By the same argument, our result is also better than the sample complexity provided in \citet{papini2018stochastic} by a factor of $O(1/\epsilon^{1/3})$. 
The key ideas in our theoretical analysis are twofold: (i) we prove a key lemma that controls the variance of importance weights introduced in SVRPG to deal with the non-stationarity of the sample distribution in reinforcement learning. This helps offset the additional variance introduced by importance sampling; and (ii) we provide a refined proof of the convergence of SVRPG and carefully investigate the trade-off between the convergence rate and computational efficiency of SG methods. This enables us to choose a smaller batch size to reduce the sample complexity while maintaining the convergence rate. In addition, we demonstrate the advantage of SVRPG over GPOMDP and validate our theoretical results on Cartpole and Mountain Car problems.


\textbf{Notation} In this paper, scalars, vectors and matrices are denoted by lower case, lower case bold face, and upper case bold face letters respectively. We use $\|\vb\|_2$ and $\|\Ab\|_2$ to denote the vector $2$-norm of a vector $\vb\in\RR^d$ and the spectral norm of a matrix $\Ab\in\RR^{d\times d}$ respectively. We denote $a_n=O(b_n)$ if $a_n\leq Cb_n$ for some constant $0<C$. For $\alpha>0$, the R\'{e}nyi divergence \citep{renyi1961measures} between distributions $P,Q$ is 
\begin{align*}
    D_{\alpha}(P||Q)=\frac{1}{\alpha-1}\log_2 \int_{x}P(x)\bigg(\frac{P(x)}{Q(x)}\bigg)^{\alpha-1}\dd x,
\end{align*}
which is non-negative for all $\alpha>0$. The exponentiated R\'{e}nyi divergence is defined as $d_{\alpha}(P||Q)=2^{D_{\alpha}(P||Q)}$.

\section{ADDITIONAL RELATED WORK}\label{sec:related}

In this section, we review additional relevant work that is not discussed in the introduction.

Deep RL models \citep{mnih2015human} have been popular in solving complex problems such as robot locomotion, playing grandmaster skill-level Go, and safe autonomous driving \citep{levine2015learning, silver2016mastering-go, shalev2016auto-driving}. 
Policy gradient \citep{sutton2000policy} is one of the most effective algorithms, 
where the policy is usually approximated by linear functions or nonlinear functions such as neural networks, and can be both stochastic and deterministic \citep{silver2014deterministic}. One major drawback of traditional policy gradient methods such as REINFORCE \citep{williams1992simple}, GPOMDP \citep{baxter2001infinite} and TRPO \citep{schulman2015trust} is the large variance caused in the estimation of the gradient \citep{sehnke2010parameter}, which leads to a poor convergence performance in practice. One way of reducing the variance in gradient estimation is to introduce various baselines as control variates \citep{weaver2001optimal,greensmith2004variance,peters2008reinforcement,gu2017q,tucker2018mirage}. \citep{pirotta2013adaptive} proposed to use adaptive step size to offset the effect of variance of the policy. \citet{papini2017adaptive} further studied the adaptive batch size used to approximate the gradient and proposed to jointly optimize the adaptive step size and batch size. It has also been extensively studied to reduce the variance of policy gradient by importance sampling \citep{liu2008monte,cortes2010learning}.  \citet{metelli2018policy} reduced the variance caused by importance sampling by deriving a surrogate objective with a Renyi penalty. 

\section{PRELIMINARIES}\label{sec:preliminary}

In this section, we introduce the preliminaries on reinforcement learning and policy gradient.\\
\textbf{Markov Decision Process:} We will model the reinforcement learning task as a discrete-time Markov Decision Process (MDP): $M=\{\cS,\cA,\cP,\cR,\gamma,\rho\}$, where $\cS$ is the state space and $\cA$
is the action space. $\cP(s'|s,a)$ defines the probability that the agent transits to state $s'$ when taking action $a$ in state $s$. The reward function $\cR(s,a):\cS\times\cA\mapsto[0,R]$ gives the reward after the agent takes action
$a$ at state $s$ for some constant $R>0$, and  $\gamma \in (0,1)$ is the discount factor. $\rho$ is the initial state distribution. The probability that the agent chooses action $a$ at state $s$ is modeled by its policy $\pi(a|s)$. Following any stationary policy, the agent can observe and collect a trajectory $\tau=\{s_0,a_0,s_1,a_1,\ldots,s_{H-1},a_{H-1},s_H\}$ which is a sequence of state-action pairs, where $H$ is the trajectory horizon. Along with the state-action pairs, the agent also observes an cumulative discounted reward
\begin{align}\label{eq:accumu_reward}
    \textstyle{\cR(\tau)=\sum_{h=0}^{H-1}\gamma^{h}\cR(s_h,a_h).}
\end{align}
\textbf{Policy Gradients:} Suppose that the policy $\pi$ is parameterized by an unknown parameter $\btheta\in\RR^d$ and denoted by $\pi_{\btheta}$. We denote the distribution induced by policy $\pi_{\btheta}$ as $p(\tau|\pi_{\btheta})$, also referred to as $p(\tau|\btheta)$ for simplicity. Then
\begin{align}\label{eq:def_distribution_tau}
    p(\tau|\btheta)=\rho(s_0)\prod_{h=0}^{H-1}\pi_{\btheta}(a_h|s_h)P(s_{h+1}|s_h,a_h).
\end{align}
To measure the performance of a given policy $\pi_{\btheta}$, we define the expected total reward under this policy as $J(\btheta)=\EE_{\tau\sim p(\cdot|\btheta)}[\cR(\tau)|M]$. Taking the gradient of $J(\btheta)$ with respect to $\btheta$ gives
\begin{align}\label{eq:def_full_grad}
    \nabla_{\btheta} J(\btheta)&=\int_{\tau}\cR(\tau)\nabla_{\btheta}p(\tau|\btheta)\dd\tau\notag\\
    &=\int_{\tau}\cR(\tau)\frac{\nabla_{\btheta}p(\tau|\btheta)}{p(\tau|\btheta)}p(\tau|\btheta)\dd\tau\notag\\
    &=\EE_{\tau\sim p(\cdot|\btheta)}[\nabla_{\btheta}\log p(\tau|\btheta) \cR(\tau)|M].
\end{align}
We can update the policy by running gradient ascent based algorithms on $\btheta$. However, it is impossible to calculate the full gradient in reinforcement learning. In particular, policy gradient samples a batch of trajectories $\{\tau_i\}_{i=1}^{N}$ to approximate the full gradient in \eqref{eq:def_full_grad}. At the $k$-th iteration, the policy is then updated by
\begin{align}\label{eq:def_gd}
    \btheta_{k+1}=\btheta_k+\eta\hat\nabla_N J(\btheta_k),
\end{align}
where $\eta>0$ is the step size and the estimated gradient $\hat\nabla_N J(\btheta_k)$ is an approximation of \eqref{eq:def_full_grad} based on trajectories $\{\tau_i\}_{i=1}^{N}$, which is defined as follows
\begin{align*}
    \hat\nabla_N J(\btheta)=\frac{1}{N}\sum_{i=1}^N\nabla_{\btheta}\log p(\tau_i|\btheta)\cR(\tau_i).
\end{align*}
According to \eqref{eq:def_distribution_tau}, we know that $\nabla_{\btheta}\log p(\tau_i|\btheta)$ is independent of the transition matrix $P$. Therefore, combining this with \eqref{eq:accumu_reward} yields
\begin{align*}
    &\hat\nabla_N J(\btheta)\notag\\
    &=\frac{1}{N}\sum_{i=1}^N\underbrace{\Bigg[\sum_{h=0}^{H-1}\nabla_{\btheta}\log \pi_{\btheta}(a_h^i|s_h^i)\Bigg]\Bigg[\sum_{h=0}^{H-1}\gamma^h\cR(s_h^i,a_h^{i})\Bigg]}_{g(\tau_i|\btheta)},
\end{align*}
where $\tau_i=\{s_0^i,a_0^i,s_1^i,a_1^i,\ldots,s_{H-1}^i,a_{H-1}^i,s_H^i\}$ for all $i=1,\ldots,N$ are sampled from policy $\pi_{\btheta}$, and $g(\tau_i|\btheta)$ is the unbiased gradient estimator based on sample $\tau_i$. Then we can rewrite the gradient in \eqref{eq:def_gd} as $\hat\nabla_N J(\btheta)=1/N\sum_{i=1}^N g(\tau_i|\btheta)$. 
Based on the above estimator, we can obtain the most well-known gradient estimators for policy gradient such as REINFORCE \citep{williams1992simple} and GPOMDP \citep{baxter2001infinite}. In particular, the REINFORCE estimator introduces an additional term $b$ as the constant baseline:
\begin{align}\label{eq:def_reinforce}
    &g(\tau_i|\btheta)\\
    &=\Bigg[\sum_{h=0}^{H-1}\nabla_{\btheta}\log \pi_{\btheta}(a_h^i|s_h^i)\Bigg]\Bigg[\sum_{h=0}^{H-1}\gamma^h\cR(s_h^i,a_h^{i})-b\Bigg].\notag
\end{align}
GPOMDP is a refined estimator of REINFORCE based on the fact that the current action does not affect previous decisions:
\begin{align}\label{eq:def_GPOMDP}
    &g(\tau_i|\btheta)\\
    &=\sum_{h=0}^{H-1}\bigg(\sum_{t=0}^{h}\nabla_{\btheta}\log \pi_{\btheta}(a_t^i|s_t^i)\bigg)\big(\gamma^h r(s_h^i,a_h^{i})-b_h\big).\notag
\end{align}

\section{ALGORITHM}\label{sec:alg}
In each iteration of the gradient ascent update \eqref{eq:def_gd}, policy gradient methods need to sample a batch of trajectories to estimate the expected gradient. This subsampling introduces a high variance and undermines the convergence speed of the algorithm. Inspired by the success of stochastic variance-reduced gradient (SVRG) techniques in stochastic optimization \citep{johnson2013accelerating,reddi2016stochastic,allen2016variance}, \citet{papini2018stochastic} proposed a stochastic variance reduced policy gradient (SVRPG) method, which is displayed in Algorithm~\ref{alg:svrg_pg}.

SVRPG consists of multiple epochs. At the beginning of the $\Oind$-th epoch, it treats the current policy as a reference point denoted by $\tbtheta^{\Oind}=\btheta_{0}^{\Oind+1}$. It then computes a gradient estimator $\mu_{\Oind}=1/N\sum_{i=1}^N g(\tau_i|\tbtheta^{\Oind})$ based on $N$ trajectories $\{\tau_i\}_{i=1}^N$ sampled from the current policy, where $g(\tau_i|\tbtheta^{\Oind})$ is the REINFORCE or GPOMDP estimator. At the $\Iind$-th iteration within the $\Oind$-th epoch, SVRPG samples $B$ trajectories $\{\tau_j\}_{j=1}^B$ based on the current policy $\btheta_{\Iind}^{\Oind+1}$. Then it updates the policy based on the following semi-stochastic gradient
\begin{align}\label{eq:def_semi_gradient}
    \vb_{\Iind}^{\Oind+1}
    &= \frac{1}{B} \textstyle{\sum^{B}_{j=1}}  g(\tau_j|\btheta_{\Iind}^{\Oind+1})\notag\\
    &\quad+\mu_{\Oind}-\frac{1}{B} \sum^{B}_{j=1}  \omega(\tau_j|\tbtheta^{\Oind},\btheta_{\Iind}^{\Oind+1}) g(\tau_j|\tbtheta^{\Oind}) , 
\end{align}
where the last two terms serve as a correction to the subsampled gradient estimator which reduces the variance and improves the convergence rate of Algorithm \ref{alg:svrg_pg}. It is worth noting that the semi-stochastic gradient in \eqref{eq:def_semi_gradient} differs from the common one used in SVRG due to the additional term $\omega(\tau|\tbtheta^{\Oind},\btheta_{\Iind}^{\Oind+1})=p(\tau|\tbtheta^{\Oind})/p(\tau|\btheta_{\Iind}^{\Oind+1})$, which is called the importance sampling weight from $p(\tau|\btheta_{\Iind}^{\Oind+1})$ to $p(\tau|\tbtheta^{\Oind})$. This term is important in reinforcement learning due to the non-stationarity of the distribution of $\tau$. Specifically, $\{\tau_i\}_{i=1}^N$ are sampled  from $\tbtheta^{\Oind}$ while $\{\tau_j\}_{j=1}^B$ are sampled based on $\btheta_{\Iind}^{\Oind+1}$. Nevertheless, we have
\begin{align*}
    \EE_{\pi_{\btheta_{\Iind}^{\Oind+1}}}\big[\omega(\cdot|\tbtheta^{\Oind},\btheta_{\Iind}^{\Oind+1}) g(\cdot|\tbtheta^{\Oind})\big]
    &=\EE_{\pi_{\tbtheta^{\Oind}}}\big[ g(\cdot|\tbtheta^{\Oind})\big],
\end{align*}
which ensures the correction term is zero mean and thus $\vb_{\Iind}^{\Oind+1}$ is an unbiased gradient estimator.


\begin{algorithm}[ht]
\caption{SVRPG} 
\label{alg:svrg_pg}
\begin{algorithmic}[1]
\STATE \textbf{Input:} number of epochs $\OindLen$, epoch size $\IindLen$, step size $\eta$, batch size $N$, mini-batch size $B$, gradient estimator $g$, initial parameter $\btheta_m^0 := \tbtheta^0 := \btheta_0$
\FOR{$\Oind=0,\ldots,\OindLen-1$}
\STATE $\btheta_{0}^{\Oind+1}=\tbtheta^{\Oind}=\btheta_{\IindLen}^{\Oind}$
\STATE Sample $N$ trajectories $\{\tau_i\}$ from $p(\cdot|\tbtheta^{\Oind})$
\STATE $\mu_{\Oind}=\hat\nabla_{N}J(\tbtheta^{\Oind}):=\frac{1}{N}\sum_{i=1}^N g(\tau_i|\tbtheta^{\Oind})$
\FOR{$\Iind=0,\ldots,\IindLen-1$}
\STATE Sample $B$ trajectories $\{\tau_j\}$ from $p(\cdot|\btheta^{\Oind+1}_{\Iind})$
\STATE $\vb^{\Oind+1}_{\Iind}=\mu_{\Oind}+\frac{1}{B}\sum_{j=1}^{B}\big(g\big(\tau_j|\btheta^{\Oind+1}_{\Iind}\big)-\omega\big(\tau_j|\tbtheta^{\Oind},\btheta^{\Oind+1}_{\Iind}\big)g\big(\tau_j|\tbtheta^{\Oind}\big)\big)$
\STATE $\btheta_{\Iind+1}^{\Oind+1}=\btheta_{\Iind}^{\Oind+1}+\eta\vb_{\Iind}^{\Oind+1}$
\ENDFOR 
\ENDFOR 
\RETURN $\btheta_{\text{out}}$: uniformly picked from $\{\btheta_{\Iind}^{\Oind}\}$ for $\Iind=0,\ldots,\IindLen;\Oind=0,\ldots,\OindLen$
\end{algorithmic} 
\end{algorithm}

\section{THEORY}\label{sec:theory}
In this section, we are going to provide a sharp analysis of Algorithm \ref{alg:svrg_pg}. We first lay down the following common assumption on the log-density of the policy function.  
\begin{assumption}\label{assump:smooth}
Let $\pi_{\btheta}(a|s)$ be the policy of an agent at state $s$. There exist constants $G,M>0$ such that the log-density of the policy function satisfies
\begin{align*}
    \|\nabla_{\btheta}\log \pi_{\btheta}(a|s)\|\leq \consLip,\quad \big\|\nabla_{\btheta}^2\log \pi_{\btheta}(a|s)\big\|_2\leq \consHes,
\end{align*}
for all $a\in\cA$ and $s\in\cS$.
\end{assumption}
In many real-world problems, we require that policy parameterization to change smoothly over time instead of drastically. Assumption \ref{assump:smooth} is an important condition in nonconvex optimization \citep{reddi2016stochastic,allen2016variance}, which guarantees the smoothness of the objective function $J(\btheta)$. Our assumption is slightly different from that in \citet{papini2018stochastic}, which assumes that $\frac{\partial}{\partial\theta_i}\log \pi_{\btheta}(a|s)$ and $\frac{\partial^2}{\partial\theta_i\partial \theta_j} \log \pi_{\btheta}(a|s)$ are upper bounded elementwisely. It can be easily verified that our Assumption \ref{assump:smooth} is milder than theirs. It should also be noted that although in reinforcement learning we make the assumptions on the parameterized policy, there is no difference in imposing the smoothness assumption on the performance function $J(\btheta)$ directly. In fact, Assumption \ref{assump:smooth} implies the following proposition on $J(\btheta)$.
\begin{proposition}\label{prop:smooth_obj}
Under Assumption \ref{assump:smooth}, $J(\btheta)$ is $L$-smooth with $L=HR(M+HG^2)/(1-\gamma)$. In addition, let $g(\tau|\btheta)$ be the REINFORCE or GPOMDP gradient estimators. Then for all $\btheta_1,\btheta_2\in\RR^d$, it holds that
\begin{align*}
    \|g(\tau|\btheta_1)-g(\tau|\btheta_2)\|_2\leq \consGradLip\|\btheta_1-\btheta_2\|_2
\end{align*}
and $\|g(\tau|\btheta)\|_2\leq \consGrad$ for all $\btheta\in\RR^d$, where $\consGradLip= HM(R+|b|)/(1-\gamma),\consGrad= HG(R+|b|)/(1-\gamma)$ and $b$ is the baseline reward.
\end{proposition}

The next assumption requires that the variance of the gradient estimator is bounded. 
\begin{assumption}\label{assump:bounded_variance}
There exists a constant $\sigma$ such that
\begin{align*}
    \Var\big(g(\tau|\btheta)\big)\leq\consVarBound, \quad\text{for all policy }\pi_{\btheta}.
\end{align*}
\end{assumption}
The above assumption is widely made in stochastic optimization. It can be easily verified for Gaussian policies with REINFORCE estimator \citep{zhao2011analysis,pirotta2013adaptive,papini2018stochastic}. 

The following assumption is needed due to the non-stationarity of the sample distribution, which is also made in \citet{papini2018stochastic}.
\begin{assumption}\label{assump:weight_variance}
There is a constant $\consVarWeight<\infty$ such that for each policy pairs encountered in Algorithm \ref{alg:svrg_pg}, it holds
\begin{align*}
    \Var(\omega(\tau|\btheta_1,\btheta_2))\leq\consVarWeight, \quad\forall\btheta_1,\btheta_2\in\RR^d, \tau\sim p(\cdot|\btheta_2).
\end{align*}
\end{assumption}

We now present our convergence result for SVRPG.
\begin{theorem}\label{thm:convergence_svrpg}
Under Assumptions \ref{assump:smooth}, \ref{assump:bounded_variance} and \ref{assump:weight_variance}. In Algorithm \ref{alg:svrg_pg}, suppose the step size $\eta\leq1/(4L)$ and epoch length $m$ and mini-batch size $B$ satisfy
\begin{align*}
    \frac{B}{m^2}\geq\frac{3(C_{\omega}\consGrad^2+\consGradLip^2)}{2L^2},
\end{align*}
where $C_{\omega}=H(2H\consLip^2+\consHes)(\consVarWeight+1)$, and $\consGradLip,\consGrad$ and $L$ are defined in Proposition \ref{prop:smooth_obj}. Then the output of Algorithm~\ref{alg:svrg_pg} satisfies
\begin{align*}
    \EE\big[\big\|\nabla J\big(\btheta_{\text{out}}\big)\big\|_2^2\big]\leq\frac{8(J(\btheta^{*})-J(\btheta_{0}))}{\eta\OindLen\IindLen}+\frac{6\sigma^2}{N},
\end{align*}
where $\btheta^*$ is the maximizer of $J(\btheta)$.
\end{theorem}

\begin{remark}
Let $T=\OindLen\IindLen$ be the total number of iterations Algorithm \ref{alg:svrg_pg} needs to achieve $\EE\big[\big\|\nabla J\big(\btheta_{\text{out}}\big)\big\|_2^2\big]\leq\epsilon$. The first term on the right hand side in Theorem \ref{thm:convergence_svrpg} gives an $O(1/T)$ convergence rate which matches that of \citet{papini2018stochastic} and the results in nonconvex optimization \citep{allen2016variance,reddi2016stochastic}. The second term $O(1/N)$ comes from the full gradient approximation at the beginning of each epoch in Algorithm \ref{alg:svrg_pg}. Compared with the result in \citet{papini2018stochastic}, Theorem \ref{thm:convergence_svrpg} does not have the additional term $O(1/B)$, which is offset by our elaborate and careful analysis of the variance of importance weights. This also enables us to choose a much smaller batch size $B$ in the inner loops of Algorithm \ref{alg:svrg_pg} and leads to a lower sample complexity. 
\end{remark}
Based on Theorem \ref{thm:convergence_svrpg}, we can calculate the total trajectory samples Algorithm \ref{alg:svrg_pg} requires to achieve $\epsilon$-precision.
\begin{corollary}\label{coro:gradient_complexity}
Under the same conditions as in Theorem \ref{thm:convergence_svrpg}, let $\epsilon>0$, if we set $\eta=1/(4L)$, $N=O(1/\epsilon)$, $B=O(1/\epsilon^{2/3})$ and $m=\sqrt{B}$, then Algorithm \ref{alg:svrg_pg} needs $O(1/\epsilon^{5/3})$ trajectories in order to achieve $\EE[\|\nabla J(\btheta_{\text{out}})\|_2^2]\leq\epsilon$.
\end{corollary}
\begin{table}[t]
\vspace{-0.1in}
\caption{Comparison on sample complexity required to achieve $\|\nabla J(\btheta)\|_2^2\leq \epsilon$.}
\label{table:complexity}
\begin{center}
\begin{tabular}{ll}
\multicolumn{1}{c}{\bf METHODS}  &\multicolumn{1}{c}{\bf  COMPLEXITY} \\
\hline \\
SG         &$O(1/\epsilon^2)$\\
SVRPG \citep{papini2018stochastic} & $O(1/\epsilon^2)$\\
SVRPG (This paper)&$O(1/\epsilon^{5/3})$ \\
\end{tabular}
\end{center}
\end{table}
\begin{remark}
In Theorem 4.4 of \citet{papini2018stochastic}, the authors showed that the sample complexity of SVRPG is $O((B+N/m)/\epsilon)$. In order to make the gradient small enough, they essentially require that $B,N=O(1/\epsilon)$, which leads to $O(1/\epsilon^2)$ sample complexity. In sharp contrast, our Corollary \ref{coro:gradient_complexity} shows that the SVRPG algorithm only needs $O(1/\epsilon^{5/3})$ number of trajectories to achieve $\|\nabla J(\btheta)\|_2^2\leq\epsilon$, which is obviously lower than the sample complexity proved in \citet{papini2018stochastic}. We present a straightforward comparison in Table \ref{table:complexity} to show the sample complexities of different methods. SG represents vanilla stochastic gradient based methods such as REINFORCE and GPOMDP. It can be seen from Table \ref{table:complexity} that our analysis yields the lowest complexity.
\end{remark}

\section{PROOF OF THE MAIN THEORY}\label{sec:proof}
In this section, we prove our main theoretical results. 

\subsection{PROOF OF MAIN THEORETICAL RESULTS}
Before we provide the proof of Theorem \ref{thm:convergence_svrpg}, we first lay down the following key lemma that controls the variance of the importance sampling weights $\omega(\tau|\tbtheta^{\Oind},\btheta_{\Iind}^{\Oind+1})$.
\begin{lemma}\label{lemma:importance_sampling_variance}
Let $\omega\big(\tau|\tbtheta^{\Oind},\btheta_{\Iind}^{\Oind+1}\big)=p(\tau|\tbtheta^{\Oind})/p(\tau|\btheta_{\Iind}^{\Oind+1})$. Under Assumptions \ref{assump:smooth} and \ref{assump:weight_variance}, it holds that
\begin{align*}
    \Var\big(\omega\big(\tau|\tbtheta^{\Oind},\btheta_{\Iind}^{\Oind+1}\big)\big)\leq C_{\omega}\|\tbtheta^{\Oind}-\btheta_{\Iind}^{\Oind+1}\|_2^2,
\end{align*}
where $C_{\omega}=H(2H\consLip^2+\consHes)(\consVarWeight+1)$.
\end{lemma}
Lemma \ref{lemma:importance_sampling_variance} shows that the variance of the importance weight is proportional to the distance between the behavioral and the target policies. Note that this upper bound could be trivial based on Assumption \ref{assump:weight_variance} when the distance is large. However, Lemma \ref{lemma:importance_sampling_variance} also provides a fine-grained control of the variance when the behavioral and target polices are sufficiently close.

Now we are ready to present the proof of our main theorem, which is also inspired from that in \citet{li2018simple}.
\begin{proof}[Proof of Theorem \ref{thm:convergence_svrpg}]
By Proposition \ref{prop:smooth_obj}, $J(\btheta)$ is $L$-smooth, which leads to
\begin{align}\label{eq:converge_smooth}
    J\big(\btheta_{\Iind+1}^{\Oind+1}\big)&\geq J\big(\btheta_{\Iind}^{\Oind+1}\big)+\big\la\nabla J\big(\btheta_{\Iind}^{\Oind+1}\big),\btheta_{\Iind+1}^{\Oind+1}-\btheta_{\Iind}^{\Oind+1}\big\ra\notag\\
    &\quad-\frac{L}{2}\big\|\btheta_{\Iind+1}^{\Oind+1}-\btheta_{\Iind}^{\Oind+1}\big\|_2^2\notag\\
    &=J\big(\btheta_{\Iind}^{\Oind+1}\big)+\big\la\nabla J\big(\btheta_{\Iind}^{\Oind+1}\big)-\vb_{\Iind}^{\Oind+1},\eta\vb_{\Iind}^{\Oind+1}\big\ra\notag\\
    &\quad+\eta\big\|\vb_{\Iind}^{\Oind+1}\big\|_2^2-\frac{L}{2}\big\|\btheta_{\Iind+1}^{\Oind+1}-\btheta_{\Iind}^{\Oind+1}\big\|_2^2\notag\\
    &\geq J\big(\btheta_{\Iind}^{\Oind+1}\big)-\frac{\eta}{2}\big\|\nabla J\big(\btheta_{\Iind}^{\Oind+1}\big)-\vb_{\Iind}^{\Oind+1}\big\|_2^2\notag\\
    &\quad+\frac{\eta}{2}\big\|\vb_{\Iind}^{\Oind+1}\big\|_2^2-\frac{L}{2}\big\|\btheta_{\Iind+1}^{\Oind+1}-\btheta_{\Iind}^{\Oind+1}\big\|_2^2\notag\\
    &\geq J\big(\btheta_{\Iind}^{\Oind+1}\big)-\frac{3\eta}{4}\big\|\nabla J\big(\btheta_{\Iind}^{\Oind+1}\big)-\vb_{\Iind}^{\Oind+1}\big\|_2^2\notag\\
    &\quad+\bigg[\frac{1}{4\eta}-\frac{L}{2}\bigg]\big\|\btheta_{\Iind+1}^{\Oind+1}-\btheta_{\Iind}^{\Oind+1}\big\|_2^2\notag\\
    &\quad+\frac{\eta}{8}\big\|\nabla J\big(\btheta_{\Iind}^{\Oind+1}\big)\big\|_2^2,
\end{align}
where the second inequality holds due to Young's inequality and the last inequality comes from the fact that $\|\nabla J(\btheta_{\Iind}^{\Oind+1})\|_2^2\leq 2\|\vb_{\Iind}^{\Oind+1}\|_2^2+2\|\nabla J(\btheta_{\Iind}^{\Oind+1})-\vb_{\Iind}^{\Oind+1}\|_2^2$. Let $\EE_{N,B}$ denote the expectation only over the randomness of the sampling trajectories $\{\tau_i\}_{i=1}^{N}$ and $\{\tau_j\}_{j=1}^{B}$
\begin{align}
    &\EE_{N,B}\big\|\nabla J\big(\btheta_{\Iind}^{\Oind+1}\big)-\vb_{\Iind}^{\Oind+1}\big\|_2^2\notag\\&=\EE_{N,B}\bigg\|\nabla J\big(\btheta_{\Iind}^{\Oind+1}\big)-\mu_s\notag\\
    &\quad+\frac{1}{B}\sum_{j=1}^B\big(\omega(\tau_j|\tbtheta^{\Oind},\btheta_{\Iind}^{\Oind+1})g\big(\tau_j|\tbtheta^{\Oind}\big)-g\big(\tau_j|\btheta_{\Iind}^{\Oind+1}\big)\big)\bigg\|_2^2\notag\\
    &=\EE_{N,B}\bigg\|\nabla J\big(\btheta_{\Iind}^{\Oind+1}\big)-\nabla J(\tbtheta^{\Oind})+\nabla J(\tbtheta^{\Oind})-\mu_s\notag\\
    &\quad+\frac{1}{B}\sum_{j=1}^B\big(\omega(\tau_j|\tbtheta^{\Oind},\btheta_{\Iind}^{\Oind+1})g\big(\tau_j|\tbtheta^{\Oind}\big)-g\big(\tau_j|\btheta_{\Iind}^{\Oind+1}\big)\big)\bigg\|_2^2\notag\\
    &=\EE_{N,B}\bigg\|\nabla J\big(\btheta_{\Iind}^{\Oind+1}\big)-\nabla J(\tbtheta^{\Oind})\notag\\
    &\quad+\frac{1}{B}\sum_{j=1}^B\big(\omega(\tau_j|\tbtheta^{\Oind},\btheta_{\Iind}^{\Oind+1})g\big(\tau_j|\tbtheta^{\Oind}\big)-g\big(\tau_j|\btheta_{\Iind}^{\Oind+1}\big)\big)\bigg\|_2^2\notag\\
    &\quad+\EE_{N,B}\bigg\|\nabla J(\tbtheta^{\Oind})-\frac{1}{N}\sum_{i=1}^Ng\big(\tau_i|\tbtheta^{\Oind}\big)\bigg\|_2^2\label{eq:grad_diff_indepence}\\
    &=\frac{1}{B^2}\sum_{j=1}^B\EE_{N,B}\big\|\nabla J\big(\btheta_{\Iind}^{\Oind+1}\big)-\nabla J(\tbtheta^{\Oind})\notag\\
    &\quad+\omega(\tau_j|\tbtheta^{\Oind},\btheta_{\Iind}^{\Oind+1})g\big(\tau_j|\tbtheta^{\Oind}\big)-g\big(\tau_j|\btheta_{\Iind}^{\Oind+1}\big)\big\|_2^2\notag\\
    &\quad+\frac{1}{N^2}\sum_{i=1}^N\EE_{N,B}\big\|\nabla J(\tbtheta^{\Oind})-g\big(\tau_i|\tbtheta^{\Oind}\big)\big\|_2^2\label{eq:grad_diff_iid_sample}\\
    &\leq\frac{1}{B^2}\sum_{j=1}^B\EE_{N,B}\big\|\omega(\tau_j|\tbtheta^{\Oind},\btheta_{\Iind}^{\Oind+1})g\big(\tau_j|\tbtheta^{\Oind}\big)\notag\\
    &\quad-g\big(\tau_j|\btheta_{\Iind}^{\Oind+1}\big)\big\|_2^2+\sigma^2/N,\label{eq:grad_diff_variance}
\end{align}
where \eqref{eq:grad_diff_indepence} holds due to the independence between trajectories $\{\tau_i\}_{i=1}^{N}$ and $\{\tau_j\}_{j=1}^{B}$, \eqref{eq:grad_diff_iid_sample} is due to $\EE\|\xb_1+\ldots+\xb_n\|_2^2=\EE\|\xb_1\|_2^2+\ldots+\EE\|\xb_n\|_2^2$ for independent and zero mean variables $\xb_1,\ldots,\xb_n$, and \eqref{eq:grad_diff_variance} follows Assumption \ref{assump:bounded_variance} and the fact that $\EE\|\xb-\EE\xb\|_2^2\leq\EE\|\xb\|_2^2$. Note that we have
\begin{align}\label{eq:grad_diff_variance_decomp}
    &\EE_{N,B}\big\|\omega(\tau_j|\tbtheta^{\Oind},\btheta_{\Iind}^{\Oind+1})g\big(\tau_j|\tbtheta^{\Oind}\big)-g\big(\tau_j|\btheta_{\Iind}^{\Oind+1}\big)\big\|_2^2\notag\\
    &\leq\EE_{N,B}\big\|\big(\omega(\tau_j|\tbtheta^{\Oind},\btheta_{\Iind}^{\Oind+1})-1\big)g\big(\tau_j|\tbtheta^{\Oind}\big)\big\|_2^2\notag\\
    &\quad+\EE_{N,B}\big\|g\big(\tau_j|\tbtheta^{\Oind}\big)-g\big(\tau_j|\btheta_{\Iind}^{\Oind+1}\big)\big\|_2^2\notag\\
    &\leq \consGrad^2\EE_{N,B}\big\|\omega(\tau_j|\tbtheta^{\Oind},\btheta_{\Iind}^{\Oind+1})-1\big\|_2^2+\consGradLip^2\big\|\tbtheta^{\Oind}-\btheta_{\Iind}^{\Oind+1}\big\|_2^2,
\end{align}
where the second inequality comes from Proposition \ref{prop:smooth_obj}. By Lemma \ref{lemma:importance_sampling_variance}, we have
\begin{align}\label{eq:grad_diff_import_sample_variance}
    &\EE_{N,B}\big\|\omega(\tau_j|\tbtheta^{\Oind},\btheta_{\Iind}^{\Oind+1})-1\big\|_2^2\notag\\
    &=\Var_{\tbtheta^{\Oind},\btheta_{\Iind}^{\Oind+1}}\big(\omega(\tau_j|\tbtheta^{\Oind},\btheta_{\Iind}^{\Oind+1}) \big)\notag\\
    &\leq C_{\omega}\big\|\btheta_{\Iind}^{\Oind+1}-\tbtheta^{\Oind}\big\|_2^2.
\end{align}
where $C_{\omega}=(2\consLip^2+\consHes)(\consVarWeight+1)$. Substituting the results in \eqref{eq:grad_diff_variance}, \eqref{eq:grad_diff_variance_decomp} and \eqref{eq:grad_diff_import_sample_variance} into \eqref{eq:converge_smooth} yields
\begin{align}\label{eq:converge_smooth_expected}
   &\EE_{N,B}\big[J\big(\btheta_{\Iind+1}^{\Oind+1}\big)\big]\notag\\ &\geq \EE_{N,B}\big[J\big(\btheta_{\Iind}^{\Oind+1}\big)\big]+\frac{\eta}{8}\EE_{N,B}\big[\big\|\nabla J\big(\btheta_{\Iind}^{\Oind+1}\big)\big\|_2^2\big]\notag\\
   &\quad+\bigg[\frac{1}{4\eta}-\frac{L}{2}\bigg]\EE_{N,B}\big[\big\|\btheta_{\Iind+1}^{\Oind+1}-\btheta_{\Iind}^{\Oind+1}\big\|_2^2\big]-\frac{3\eta\sigma^2}{4N}\notag\\
   &\quad-\frac{3\eta(C_{\omega}\consGrad^2+L^2)}{4B}\EE_{N,B}\big[\big\|\btheta_{\Iind}^{\Oind+1}-\tbtheta^{\Oind}\big\|_2^2\big].
\end{align}
For the ease of notation, we denote
\begin{align}\label{def:constant_variance}
    \consVariance=\frac{3(C_{\omega}\consGrad^2+\consGradLip^2)}{4B}.
\end{align}
By Young's inequality (Peter-Paul inequality), we have
\begin{align*}
   \big\|\btheta_{\Iind+1}^{\Oind+1}-\tbtheta^{\Oind}\big\|_2^2&\leq(1+\alpha) \big\|\btheta_{\Iind+1}^{\Oind+1}-\btheta_{\Iind}^{\Oind+1}\big\|_2^2\\
   &\quad+(1+1/\alpha)\big\|\btheta_{\Iind}^{\Oind+1}-\tbtheta^{\Oind}\big\|_2^2
\end{align*}
holds for any $\alpha>0$. For $\eta\leq 1/(2L)$, combining the above inequality with \eqref{eq:converge_smooth_expected} and \eqref{def:constant_variance} yields
\begin{align*}
&\EE_{N,B}\big[J\big(\btheta_{\Iind+1}^{\Oind+1}\big)\big] \\
&\geq \EE_{N,B}\big[J\big(\btheta_{\Iind}^{\Oind+1}\big)\big]+\frac{\eta}{8}\EE_{N,B}\big[\big\|\nabla J\big(\btheta_{\Iind}^{\Oind+1}\big)\big\|_2^2\big]-\frac{3\eta\sigma^2}{4N}\notag\\
&\quad+\frac{1}{1+\alpha}\bigg[\frac{1}{4\eta}-\frac{L}{2}\bigg]\EE_{N,B}\big[\big\|\btheta_{\Iind+1}^{\Oind+1}-\tbtheta^{\Oind}\big\|_2^2\big]\notag\\
&\quad-\bigg[\eta\consVariance+\frac{1}{\alpha}\bigg[\frac{1}{4\eta}-\frac{L}{2}\bigg]\bigg]\EE_{N,B}\big[\big\|\btheta_{\Iind}^{\Oind+1}-\tbtheta^{\Oind}\big\|_2^2\big]
\end{align*}
Now we set $\alpha=2\Iind+1$ and sum up the above inequality over $t=0,\ldots,\IindLen-1$. Note that $\btheta_{0}^{\Oind+1}=\tbtheta^{\Oind}$, $\btheta_{\IindLen}^{\Oind+1}=\tbtheta^{\Oind+1}$. We are able to obtain
\begin{align}\label{eq:converge_sum_over_t}
    &\EE_{N}\big[J\big(\tbtheta^{\Oind+1}\big)\big] \notag\\
    &\geq \EE_{N}\big[J\big(\tbtheta^{\Oind}\big)\big]+\frac{\eta}{8}\sum_{\Iind=0}^{\IindLen-1}\EE_{N}\big[\big\|\nabla J\big(\btheta_{\Iind}^{\Oind+1}\big)\big\|_2^2\big]-\frac{3\IindLen\eta\sigma^2}{4N}\notag\\
    &\quad+\sum_{t=0}^{\IindLen-1}\frac{1/(2\eta)-L}{4(t+1)}\EE_{N}\big[\big\|\btheta_{\Iind+1}^{\Oind+1}-\tbtheta^{\Oind}\big\|_2^2\big]\notag\\
    &\quad-\sum_{t=0}^{\IindLen-1}\bigg[\eta\consVariance+\frac{1/(2\eta)-L}{2(2t+1)}\bigg]\EE_{N}\big[\big\|\btheta_{\Iind}^{\Oind+1}-\tbtheta^{\Oind}\big\|_2^2\big]\notag\\
    &= \EE_{N}\big[J\big(\tbtheta^{\Oind}\big)\big]+\frac{\eta}{8}\sum_{t=0}^{\IindLen-1}\EE_{N}\big[\big\|\nabla J\big(\btheta_{\Iind}^{\Oind+1}\big)\big\|_2^2\big]-\frac{3\IindLen\eta\sigma^2}{4N}\notag\\
    &\quad+\sum_{t=0}^{\IindLen-2}\frac{1/(2\eta)-L}{4(t+1)}\EE_{N}\big[\big\|\btheta_{\Iind+1}^{\Oind+1}-\tbtheta^{\Oind}\big\|_2^2\big]\notag\\
    &\quad-\sum_{t=1}^{\IindLen-1}\bigg[\eta\consVariance+\frac{1/(2\eta)-L}{2(2t+1)}\bigg]\EE_{N}\big[\big\|\btheta_{\Iind}^{\Oind+1}-\tbtheta^{\Oind}\big\|_2^2\big]\notag\\
    &\quad+\frac{1/(2\eta)-L}{4\IindLen}\EE_{N}\big[\big\|\btheta_{\IindLen}^{\Oind+1}-\tbtheta^{\Oind}\big\|_2^2\big]\notag\\
    &\quad-\bigg[\eta\consVariance+\frac{1}{4\eta}-\frac{L}{2}\bigg]\EE_{N}\big[\big\|\btheta_{0}^{\Oind+1}-\tbtheta^{\Oind}\big\|_2^2\big]\notag\\
    &=\EE_{N}\big[J\big(\tbtheta^{\Oind}\big)\big]+\frac{\eta}{8}\sum_{t=0}^{\IindLen-1}\EE_{N}\big[\big\|\nabla J\big(\btheta_{\Iind}^{\Oind+1}\big)\big\|_2^2\big]-\frac{3\IindLen\eta\sigma^2}{4N}\notag\\
    &\quad+\sum_{t=1}^{\IindLen-1}\bigg[\frac{1/(4\eta)-L/2}{2t(2t+1)}-\eta\consVariance\bigg]\EE_{N}\big[\big\|\btheta_{\Iind}^{\Oind+1}-\tbtheta^{\Oind}\big\|_2^2\big]\notag\\
    &\quad+\frac{1/(2\eta)-L}{4\IindLen}\EE_{N}\big[\big\|\btheta_{\IindLen}^{\Oind+1}-\tbtheta^{\Oind}\big\|_2^2\big].
\end{align}
Recall the definition of $\consVariance$ in \eqref{def:constant_variance}. If we set step size $\eta$ and the epoch length $B$ to satisfy
\begin{align}
    \eta\leq\frac{1}{4L},\quad \frac{B}{m^2}\geq\frac{3(C_{\omega}\consGrad^2+\consGradLip^2)}{2L^2},
\end{align}
then \eqref{eq:converge_sum_over_t} leads to
\begin{align*}
    \EE_{N}\big[J\big(\tbtheta^{\Oind+1}\big)\big]
    &\geq\EE_{N}\big[J\big(\tbtheta^{\Oind}\big)\big]-\frac{3\IindLen\eta\sigma^2}{4N}\notag\\
    &\quad+\frac{\eta}{8}\sum_{t=0}^{\IindLen-1}\EE_{N}\big[\big\|\nabla J\big(\btheta_{\Iind}^{\Oind+1}\big)\big\|_2^2\big].
\end{align*}
Telescoping the above inequality yields
\begin{align*}
    &\frac{\eta}{8}\sum_{\Oind=0}^{\OindLen-1}\sum_{\Iind=0}^{\IindLen-1}\EE\big[\big\|\nabla J\big(\btheta_{\Iind}^{\Oind+1}\big)\big\|_2^2\big]\\
    &\leq\EE\big[J\big(\tbtheta^{\OindLen}\big)\big]-\EE\big[J\big(\tbtheta^{0}\big)\big]+\frac{3\OindLen\IindLen\eta\sigma^2}{4N},
\end{align*}
which immediately implies
\begin{align*}
    \EE\big[\big\|\nabla J\big(\btheta_{\text{out}}\big)\big\|_2^2\big]&\leq\frac{8\big(\EE\big[J\big(\tbtheta^{\OindLen}\big)\big]-\EE\big[J\big(\tbtheta^{0}\big)\big]\big)}{\eta\OindLen\IindLen}+\frac{6\sigma^2}{N}\\
    &\leq\frac{8(J(\btheta^{*})-J(\btheta_{0}))}{\eta\OindLen\IindLen}+\frac{6\sigma^2}{N}.
\end{align*}
This completes the proof.
\end{proof}


\begin{proof}[Proof of Corollary \ref{coro:gradient_complexity}]
By Theorem \ref{thm:convergence_svrpg}, in order to ensure $\EE\big[\big\|\nabla J\big(\btheta_{\text{out}}\big)\big\|_2^2\big]\leq\epsilon$, it suffices to ensure
\begin{align*}
    \frac{8(J(\btheta^{*})-J(\btheta_{0}))}{\eta\OindLen\IindLen}=\frac{\epsilon}{2},\quad\frac{6\sigma^2}{N}=\frac{\epsilon}{2},
\end{align*}
which implies $\OindLen\IindLen=O(1/\epsilon)$ and $N=O(1/\epsilon)$. Note that we have set $\IindLen=O(\sqrt{B})$. The total number of stochastic gradient evaluations $\cT_g$ we need is
\begin{align*}
    \cT_g=\OindLen N+\OindLen\IindLen B=O\bigg(\frac{N}{\sqrt{B}\epsilon}+\frac{B}{\epsilon}\bigg)=O\bigg(\frac{1}{\epsilon^{5/3}}\bigg),
\end{align*}
where we set $B=N^{2/3}=1/\epsilon^{2/3}$.
\end{proof}

\subsection{PROOF OF TECHNICAL LEMMAS}
In this subsection, we provide the proofs of the technical lemmas used in the proof of main theory. We first prove the smoothness of $J(\btheta)$.
\begin{proof}[Proof of Proposition \ref{prop:smooth_obj}]
Recall the notion in \eqref{eq:def_full_grad} as
\begin{align*}
    \nabla J(\btheta)&=\int_{\tau}\cR(\tau)\nabla_{\btheta}p(\tau|\btheta)\dd\tau,
\end{align*}
which directly implies the Hessian matrix 
\begin{align}\label{eq:def_hessian_den}
    \nabla^2 J(\btheta)&=\int_{\tau}\cR(\tau)\nabla_{\btheta}^2p(\tau|\btheta)\dd\tau.
\end{align}
Note that the Hessian of the log-density function is
\begin{align}\label{eq:variance_hes_logDen}
    \nabla_{\btheta}^2\log p(\tau|\btheta)&=-p(\tau|\btheta)^{-2}\nabla_{\btheta}p(\tau|\btheta)\nabla_{\btheta}p(\tau|\btheta)^{\top}\notag\\
    &\quad+p(\tau|\btheta)^{-1}\nabla_{\btheta}^2p(\tau|\btheta).
\end{align}
Substituting \eqref{eq:variance_hes_logDen} into \eqref{eq:def_hessian_den} yields
\begin{align*}
    \nabla^2 J(\btheta)&=\int_{\tau}p(\tau|\btheta)\cR(\tau)\big[\nabla_{\btheta}^2\log p(\tau|\btheta)\notag\\
    &\quad+\nabla_{\btheta}\log p(\tau|\btheta)\nabla_{\btheta}\log p(\tau|\btheta)^{\top}\big]\dd\tau.
\end{align*}
Therefore, we have
\begin{align}\label{eq:hessian_specNorm}
    \|\nabla^2 J(\btheta)\|_2&\leq\int_{\tau}p(\tau|\btheta)\cR(\tau)\big[\|\nabla_{\btheta}^2\log p(\tau|\btheta)\|_2\notag\\
    &\quad+\|\nabla_{\btheta}\log p(\tau|\btheta)\|_2^2\big]\dd\tau\notag\\
    &\leq\int_{\tau}p(\tau|\btheta)\cR(\tau)(HM+H^2G^2)\dd\tau.
\end{align}
By \eqref{eq:accumu_reward}, we have for any $\tau$ it holds that
\begin{align*}
    \cR(\tau)\leq\frac{R(1-\gamma^H)}{1-\gamma}\leq\frac{R}{1-\gamma}.
\end{align*}
Combining this with \eqref{eq:hessian_specNorm} yields
\begin{align*}
    \|\nabla^2 J(\btheta)\|_2\leq RH(M+HG^2)/(1-\gamma),
\end{align*}
which means $J(\btheta)$ is $L$-smooth with $L=RH(M+HG^2)/(1-\gamma)$. Recall the REINFORCE estimator in \eqref{eq:def_reinforce}: \begin{align*}
    g(\tau|\btheta)=\Bigg[\sum_{t=0}^{H-1}\nabla\log\pi_{\btheta}(a_t|s_t)\Bigg]\Bigg[\sum_{t=0}^{H-1}\gamma^t\cR(s_t,a_t)-b\Bigg],
\end{align*}
where $b$ is a constant baseline reward. Then we have
\begin{align*}
    \|\nabla g(\tau|\btheta)\|_2&\leq\Bigg[\sum_{t=0}^{H-1}\big\|\nabla^2\log\pi_{\btheta}(a_t|s_t)\big\|_2\Bigg]\frac{R+|b|}{1-\gamma}\\
    &\leq \frac{HM(R+|b|)}{1-\gamma}.
\end{align*}
Similarly, we have
\begin{align*}
    \|g(\tau|\btheta)\|_2\leq HG\bigg[\frac{R(1-\gamma^H)}{1-\gamma}+|b|\bigg]\leq\frac{HG(R+|b|)}{1-\gamma}.
\end{align*}
The proof of the GPOMDP estimator is similar and we omit it for simplicity. This completes the proof.
\end{proof}
The analysis of Lemma \ref{lemma:importance_sampling_variance} relies on the following important properties of importance sampling weights.
\begin{lemma}[Lemma 1 in \citet{cortes2010learning}]\label{lemma:importance_weight_property}
Let $\omega(x)=P(x)/Q(x)$ be the importance weight for distributions $P$ and $Q$. Then the following identities hold:
\begin{align*}
    \EE[\omega]&=1,\quad\EE[\omega^2]=d_2(P||Q),
\end{align*}
where $d_{2}(P||Q)=2^{D_{2}(P||Q)}$ and $D_{2}(P||Q)$ is the R\'{e}nyi divergence between distributions $P$ and $Q$. Note that this immediately implies $\Var(\omega)=d_2(P||Q)-1$.
\end{lemma}

\begin{proof}[Proof of Lemma \ref{lemma:importance_sampling_variance}]
According to Lemma \ref{lemma:importance_weight_property}, we have
\begin{align*}
    \Var\big(\omega\big(\tau|\tbtheta^{\Oind},\btheta_{\Iind}^{\Oind+1}\big)\big)=d_2\big(p(\tau|\tbtheta^{\Oind})||p(\tau|\btheta_{\Iind}^{\Oind+1})\big)-1.
\end{align*}
In the rest of this proof, we denote $\btheta_1=\tbtheta^{\Oind}$ and $\btheta_2=\btheta_{\Iind}^{\Oind+1}$ to simplify the notation. By definition, we have
\begin{align*}
    d_2(p(\tau|\btheta_1)||p(\tau|\btheta_2))&=\int_{\tau}p(\tau|\btheta_1)\frac{p(\tau|\btheta_1)}{p(\tau|\btheta_2)}\dd \tau\\
    &=\int_{\tau}p(\tau|\btheta_1)^2p(\tau|\btheta_2)^{-1}\dd\tau.
\end{align*}
For any fixed $\btheta_2\in\RR^d$, computing the gradient of $d_2(p(\tau|\btheta_1)||p(\tau|\btheta_2))$ with respect to $\btheta_1$ yields
\begin{align*}
    &\nabla_{\btheta_1}d_2(p(\tau|\btheta_1)||p(\tau|\btheta_2))\\
    &=2\int_{\tau}p(\tau|\btheta_1)\nabla_{\btheta_1}p(\tau|\btheta_1)p(\tau|\btheta_2)^{-1}\dd\tau,
\end{align*}
which implies that if we set $\btheta_1=\btheta_2$, we will obtain
\begin{align*}
   &\nabla_{\btheta_1}d_2(p(\tau|\btheta_1)||p(\tau|\btheta_2))\big|_{\btheta_1=\btheta_2}\\
   &=2\int_{\tau}\nabla_{\btheta_1}p(\tau|\btheta_1)\dd\tau\big|_{\btheta_1=\btheta_2}\\
   &=0.
\end{align*}
Hence, applying mean value theorem, we have
\begin{align}\label{eq:variance_taylor_expansion}
    &d_2(p(\tau|\btheta_1)||p(\tau|\btheta_2))\\
    &=1+\frac{1}{2}(\btheta_1-\btheta_2)^{\top}\nabla_{\btheta}^2d_2(p(\tau|\btheta)||p(\tau|\btheta_2))(\btheta_1-\btheta_2),\notag
\end{align}
where $\btheta=t\btheta_1+(1-t)\btheta_2$ for some $t\in[0,1]$. Next, we compute the Hessian matrix. For any fixed $\btheta_2$, we have
\begin{align}\label{eq:variance_hes_renyiDiv}
    &\nabla_{\btheta}^2d_2(p(\tau|\btheta)||p(\tau|\btheta_2))\notag\\
    &=2\int_{\tau}\nabla_{\btheta}p(\tau|\btheta)\nabla_{\btheta}p(\tau|\btheta)^{\top}p(\tau|\btheta_2)^{-1}\dd\tau\notag\\
    &\quad+2\int_{\tau}\nabla^2_{\btheta}p(\tau|\btheta)p(\tau|\btheta)p(\tau|\btheta_2)^{-1}\dd\tau\notag\\
    &=2\int_{\tau}\nabla_{\btheta}\log p(\tau|\btheta)\nabla_{\btheta}\log p(\tau|\btheta)^{\top}\frac{p(\tau|\btheta)^2}{p(\tau|\btheta_2)}\dd\tau\notag\\
    &\quad+2\int_{\tau}\nabla^2_{\btheta}p(\tau|\btheta)p(\tau|\btheta)p(\tau|\btheta_2)^{-1}\dd\tau.
\end{align}
Recall the Hessian of the log-density function in \eqref{eq:variance_hes_logDen}. Substituting \eqref{eq:variance_hes_logDen} into \eqref{eq:variance_hes_renyiDiv} yields
\begin{align*}
    &\|\nabla_{\btheta}^2d_2(p(\tau|\btheta)||p(\tau|\btheta_2))\|_2\\
    &=\bigg\|4\int_{\tau}\nabla_{\btheta}\log p(\tau|\btheta)\nabla_{\btheta}\log p(\tau|\btheta)^{\top}\frac{p(\tau|\btheta)^2}{p(\tau|\btheta_2)}\dd\tau\notag\\
    &\quad+2\int_{\tau}\nabla^2_{\btheta}\log p(\tau|\btheta)\frac{p(\tau|\btheta)^2}{p(\tau|\btheta_2)}\dd\tau\bigg\|_2\\
    &\leq \int_{\tau}\frac{p(\tau|\btheta)^2}{p(\tau|\btheta_2)}\big(4\|\nabla_{\btheta}\log p(\tau|\btheta)\|_2^2\\
    &\quad+2\|\nabla^2_{\btheta}\log p(\tau|\btheta)\|_2\big)\dd\tau\\
    &\leq(4H^2\consLip^2+2H\consHes)\EE[\omega(\tau|\btheta,\btheta_2)^2]\\
    &\leq2H(2H\consLip^2+\consHes)(\consVarWeight+1),
\end{align*}
where the second inequality comes from Assumption \ref{assump:smooth} and the last inequality is due to Assumption \ref{assump:weight_variance} and Lemma \ref{lemma:importance_weight_property}. Therefore, by \eqref{eq:variance_taylor_expansion} we have
\begin{align*}
    \Var\big(\omega\big(\tau|\tbtheta^{\Oind},\btheta_{\Iind}^{\Oind+1}\big)\big)&=d_2\big(p(\tau|\tbtheta^{\Oind})||p(\tau|\btheta_{\Iind}^{\Oind+1})\big)-1\\
    &\leq C_{\omega}\|\tbtheta^{\Oind}-\btheta_{\Iind}^{\Oind+1}\|_2^2,
\end{align*}
where $C_{\omega}=H(2H\consLip^2+\consHes)(\consVarWeight+1)$.
\end{proof}

\section{EXPERIMENTS}\label{sec:experiment}
In this section, we conduct experiments on reinforcement learning benchmark tasks, i.e., the Cartpole and Mountain Car (continuous) environments \citep{opengymai2016}, to evaluate the performance of Algorithm \ref{alg:svrg_pg}. We measure the performance of an algorithm in terms of the total sample trajectories it needs to achieve a certain reward. We compare SVRPG with vanilla stochastic gradient based algorithms: the REINFORCE \citep{williams1992simple} and GPOMDP\footnote{We thank \citet{papini2018stochastic} for their implementations of GPOMDP and SVRPG as well as \citet{duan2016benchmarking} for their implementations from the \textit{rllab} library.} \citep{baxter2001infinite} algorithms. Recall that at each iteration of Algorithm \ref{alg:svrg_pg}, we also need to choose certain stochastic gradient estimator to approximate the full gradient based on sampled trajectories. Since the performance of GPOMDP is always comparable or better than REINFORCE, we only report the results of SVRPG with the GPOMDP estimator. 

We follow the practical suggestions provided in  \citet{papini2018stochastic} to improve the performance including (1) performing one initial gradient update immediately after sampling the $N$ trajectories in the outer loop; (2) using adaptive step sizes; and (3) using adaptive epoch length (terminate the inner loop update early if the step size used in the inner loop is smaller than that used in the outer loop). 
Following \citet{papini2018stochastic}, we use the following Gaussian policy with a fixed standard deviation $\tilde\sigma^2$:
\begin{align*}
    \pi_{\btheta}(a|s)=1/\sqrt{2\pi}\sigma\exp\big(-(\btheta^{\top}\phi(s)-a)^2/2\tilde\sigma^2\big),
\end{align*}
where $\phi:\cS\mapsto\RR^d$ is a bounded feature map. Under the Gaussian policy, it is easy to verify that Assumptions \ref{assump:smooth} and \ref{assump:bounded_variance} is satisfied with parameters depending on $\phi,\tilde\sigma^2$ and the upper bound of the action $a$ for all $a\in\cA$.

\begin{figure}[t]
    \label{figure-exp-cartpole}
    \centering
    \subfigure[Cartpole] {\includegraphics[scale=0.27]{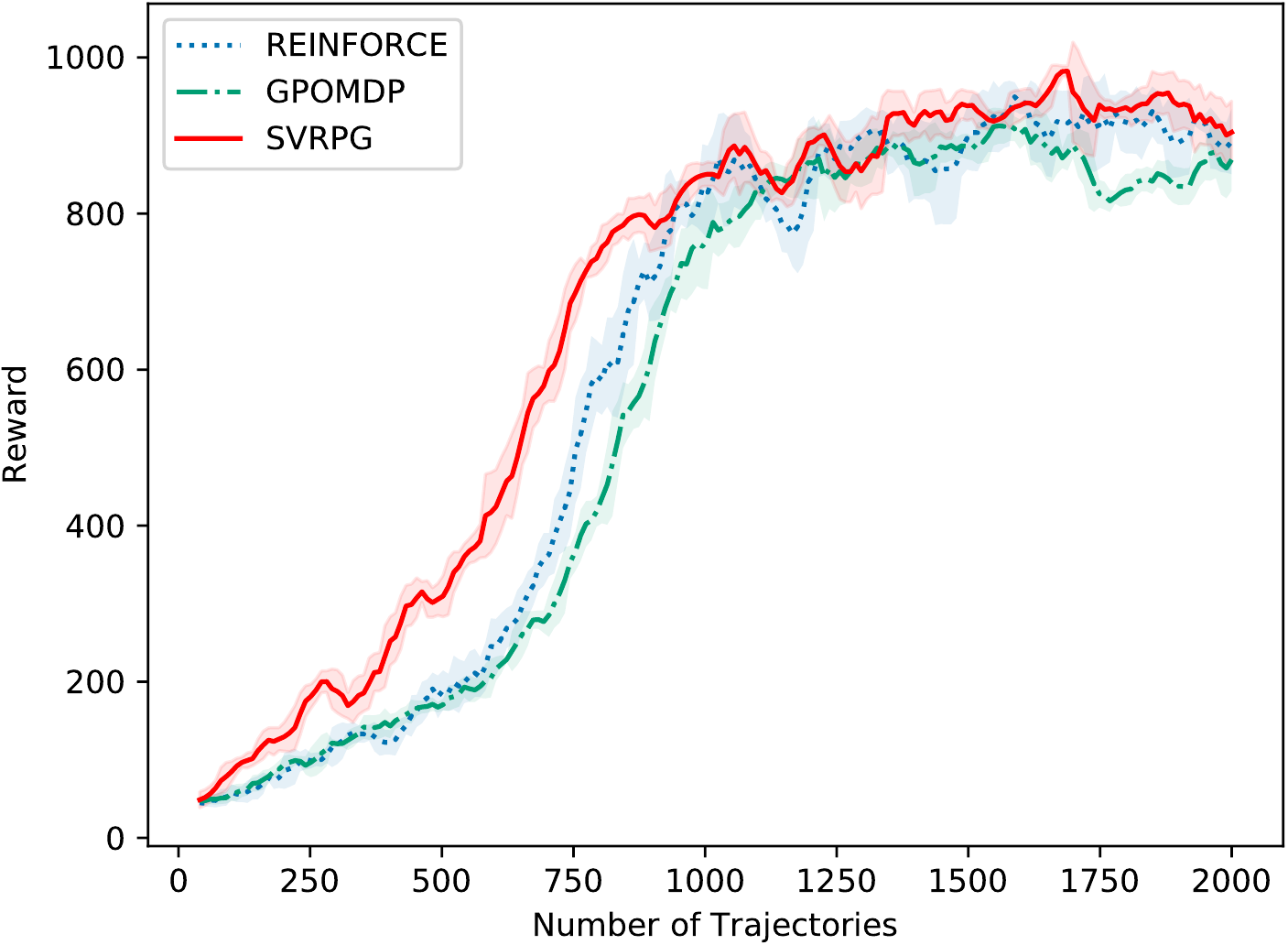}\label{fig:cartpole_methods}}
    \subfigure[Mountain Car]
    {\includegraphics[scale=0.27]{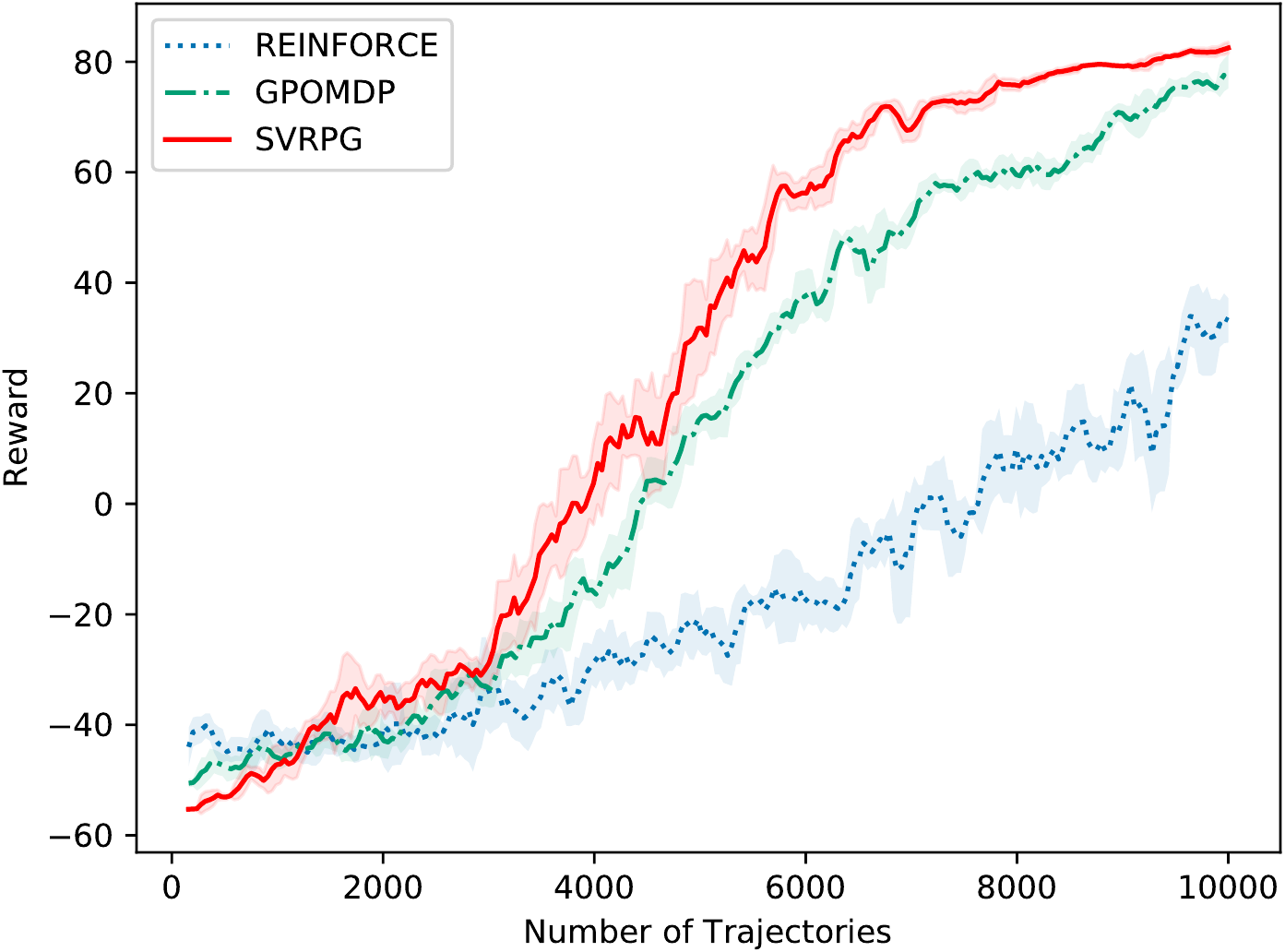}\label{fig:mountain_methods}}
    \caption{The average reward of different algorithms in Cartpole and Mountain Car environments.}
\end{figure}
\begin{figure}[t]
    \label{figure-vanilla}
    \centering
    \subfigure[Cartpole] {\includegraphics[scale=0.27]{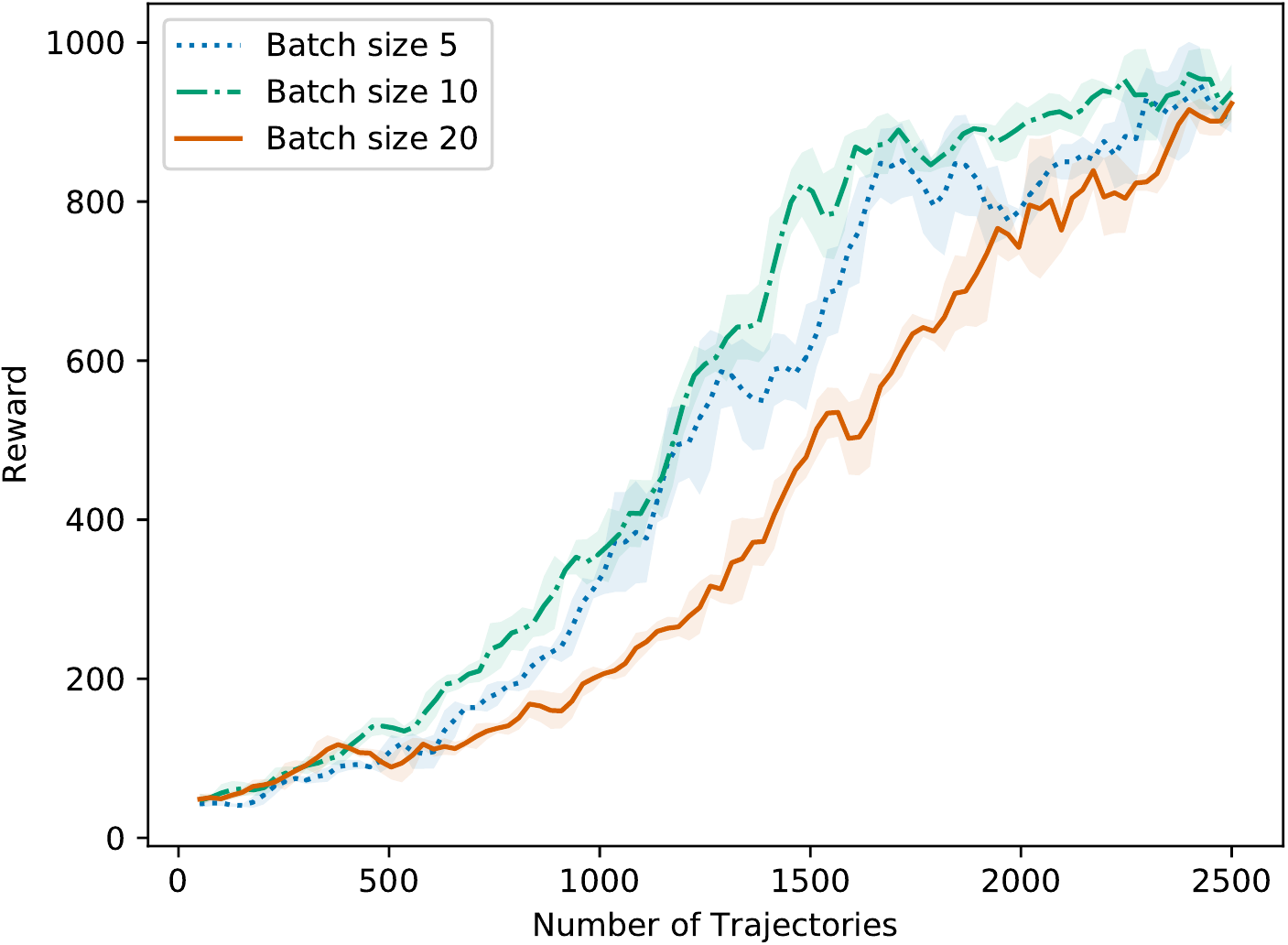}\label{fig:cartpole_batch}}
    \subfigure[Mountain Car]
    {\includegraphics[scale=0.27]{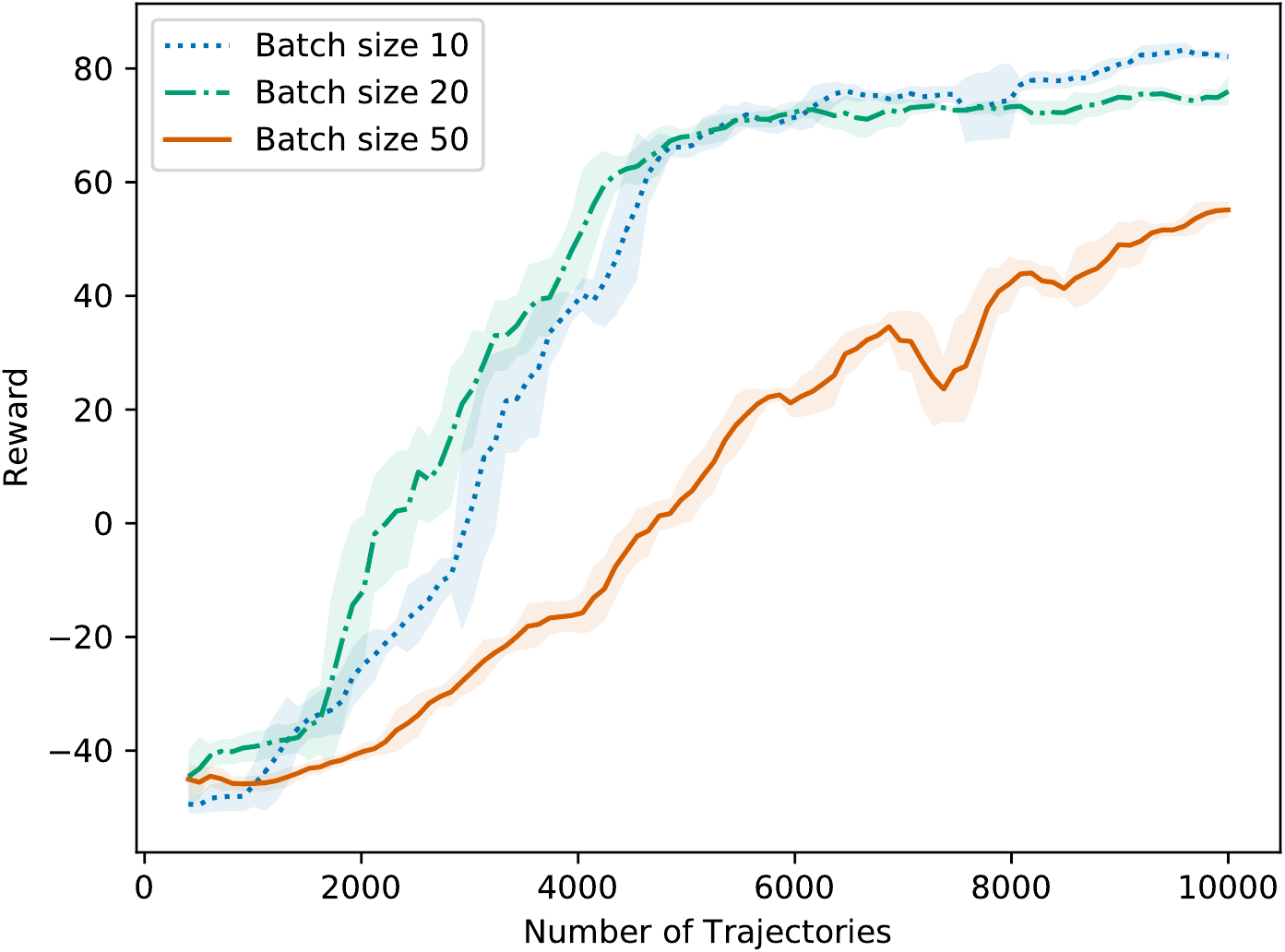} \label{fig:mountain_batch}}
    \caption{
    The average reward of SVRPG with different mini-batch sizes $B$. }
\end{figure}

\textbf{Cartpole Setup:} 
The neural network of the Cartpole environment has one hidden layer of $8$ nodes with the $tanh$ activation function. In the comparison between REINFORCE, GPOMDP, and SVRPG, we use learning rate $\eta = 0.01,0.01$ and $0.06$ for them respectively. Based on our theoretical analysis, we chose $N = 25$ and $B = 10$ for SVRPG. For the best comparison between the algorithms, we also the set the batch size of vanilla gradient methods to be $N = 10$. 

We also test the effectiveness of different mini-batch sizes within each epoch of SVRPG to validate our theoretical claims in Corollary \ref{coro:gradient_complexity}. 
We fix $N = 25$ and vary mini-batch sizes of $B = [5, 10, 20]$ respectively. As mini-batch size increases, we also scale learning rate proportionally such that $\eta = [0.01, 0.02, 0.03]$, corresponding to their respective mini-batch size. 

\textbf{Mountain Car Setup:}
The neural network for the Mountain Car environment contains one hidden layer with $16$ nodes with the $tanh$ activation. In the comparison among  REINFORCE, GPOMDP and SVRPG, we set $N = 100$ and $B = 20$ for SVRPG and set batch size $N = 20$ for the vanilla gradient methods. REINFORCE, GPOMDP, and SVRPG have respective learning rates of $\eta = [0.0025, 0.005, 0.0075]$. 

Similar to the experiments on Cartpole, we also investigate the effect of difference mini-batch sizes on SVRPG for Mountain Car. We conduct experiments on SVRPG by setting $N = 100$ and $B = [10, 20, 50]$ with corresponding learning rates of $\eta = [0.01, 0.01, 0.015]$.

\textbf{Experimental Results:}
Figures \ref{fig:cartpole_methods} and \ref{fig:mountain_methods} respectively show the performance of different algorithms on the Cartpole and Mountain Car environments. All the results are averaged over $10$ repetitions and the shaded area is a confidence interval corresponding to the standard deviation over different runs. It can be seen that all the methods solved the Cartpole environment (with averaged reward close to $1000$). The SVRPG algorithm outperforms the other two by gaining higher rewards with fewer sample trajectories. SVRPG also beats the other methods in solving the Mountain Car environment (with averaged reward close to $90$). However, the REINFORCE algorithm fails to solve the Mountain Car environment due to its high variance.



Figures \ref{fig:cartpole_batch} and \ref{fig:mountain_batch} show the effect of different mini-batch sizes $B$ on SVRPG. Note that the outer loop batch sizes of Cartpole and Mountain Car are $N=25$ and $N=100$. It can be seen that when $B=10$ and $B=20$ for Cartpole and Mountain Car respectively, SVRPG achieve the best performance, which is well aligned with our theory. In particular, with a small mini-batch size, SVRPG acts similarly to the vanilla stochastic gradient based algorithms which needs fewer trajectories in each iteration but converges slowly and requires more trajectories in total. Conversely, using a large mini-batch pushes SVRPG to converge in fewer iterations, but requires more trajectories in total. 

\section{CONCLUSIONS}\label{sec:conclusion}
We revisited the SVRPG algorithm \citep{papini2018stochastic} and derived a sharp convergence analysis of SVRPG which achieves the state-of-the-art sample complexity. We provided a detailed discussion and guidance on the choice of batch sizes and epoch length based on our improved analysis so that the total number of samples can be significantly reduced. We also empirically validated the theoretical results on common reinforcement learning tasks. As a future direction, it would be interesting to see whether any better sample complexity can be obtained for policy gradient algorithms.

\textbf{Acknowledgements}


We would like to thank the anonymous reviewers for their helpful comments. This research was sponsored in part by the National Science Foundation IIS-1904183 and IIS-1906169. The views and conclusions contained in this paper are those of the authors and should not be interpreted as representing any funding agencies.

{
\bibliographystyle{icml2019}
\bibliography{RL,optimization}
}
\end{document}